\documentclass[twoside]{article}
\usepackage{aistats2020}
\usepackage[round]{natbib}

\usepackage{times}

\usepackage{amsmath}
\usepackage{amssymb}
\usepackage{hyperref}
\usepackage{amsthm}

\usepackage{bm}
\usepackage{wrapfig}
\usepackage{xcolor,colortbl}
\usepackage{enumitem}
\usepackage{thmtools,thm-restate}
\usepackage[export]{adjustbox}
\usepackage{verbatim}
\usepackage{multirow}
\usepackage{pifont}
\newcommand{\xmark}{\ding{55}}%
\usepackage[most]{tcolorbox}

\definecolor{blued}{RGB}{70,197,221}
\definecolor{pearOne}{HTML}{2C3E50}
\definecolor{pearTwo}{HTML}{A9CF54}
\definecolor{pearTwoT}{HTML}{C2895B}
\definecolor{pearThree}{HTML}{E74C3C}
\colorlet{titleTh}{pearOne}
\colorlet{bull}{pearTwo}
\definecolor{pearcomp}{HTML}{B97E29}
\definecolor{pearFour}{HTML}{588F27}
\definecolor{pearFith}{HTML}{ECF0F1}
\definecolor{pearDark}{HTML}{2980B9}
\definecolor{pearDarker}{HTML}{1D2DEC}

\usepackage[colorinlistoftodos, textwidth=20mm]{todonotes}

\let\originalleft\left
\let\originalright\right
\renewcommand{\left}{\mathopen{}\mathclose\bgroup\originalleft}
\renewcommand{\right}{\aftergroup\egroup\originalright}

\definecolor{graphicbackground}{rgb}{0.96,0.96,0.8}
\definecolor{rouge1}{RGB}{226,0,38}  
\definecolor{orange1}{RGB}{243,154,38}  
\definecolor{jaune}{RGB}{254,205,27}  
\definecolor{blanc}{RGB}{255,255,255} 
\definecolor{rouge2}{RGB}{230,68,57}  
\definecolor{orange2}{RGB}{236,117,40}  
\definecolor{taupe}{RGB}{134,113,127} 
\definecolor{gris}{RGB}{91,94,111} 
\definecolor{bleu1}{RGB}{38,109,131} 
\definecolor{bleu2}{RGB}{28,50,114} 
\definecolor{vert1}{RGB}{133,146,66} 
\definecolor{vert3}{RGB}{20,200,66} 
\definecolor{vert2}{RGB}{157,193,7} 
\definecolor{darkyellow}{RGB}{233,165,0}  
\definecolor{lightgray}{rgb}{0.9,0.9,0.9}
\definecolor{darkgray}{rgb}{0.6,0.6,0.6}
\definecolor{babyblue}{rgb}{0.54, 0.81, 0.94}
\definecolor{citrine}{rgb}{0.89, 0.82, 0.04}
\definecolor{misogreen}{rgb}{0.25,0.6,0.0}

\DeclareMathOperator*{\argmax}{arg\,max}

\renewcommand{\d}[1]{\ensuremath{\operatorname{d}\!{#1}}}

\newcommand{\I}{{\mathds{1}}}

\newtheorem{definition}{Definition}
\newcommand{\R}{\mathbb{R}}

\newcommand{\E}{\mathbb{E}}

\newcommand{\CommaBin}{\mathbin{\raisebox{0.5ex}{,}}}

\newcommand{\cO}{\mathcal{O}}
\newcommand{\tcO}{\widetilde{\cO}}

\newcommand{\cX}{\mathcal{X}}
\newcommand{\X}{\cX}

\renewcommand{\epsilon}{\varepsilon}
\renewcommand{\hat}{\widehat}
\renewcommand{\tilde}{\widetilde}
\renewcommand{\bar}{\overline}

\newcommand{\nothere}[1]{}

\usepackage{xspace}

\newcommand{\Vroom}{\normalfont{\texttt{VROOM}}\xspace}
\newcommand{\SequOOL}{\texttt{SequOOL}\xspace}
\newcommand{\StroquOOL}{\texttt{StroquOOL}\xspace}

\newcommand{\StoSOO}{\texttt{StoSOO}\xspace}
\newcommand{\POO}{\texttt{POO}\xspace}
\newcommand{\DOO}{\texttt{DOO}\xspace}
\newcommand{\SOO}{\texttt{SOO}\xspace}
\newcommand{\GPO}{\texttt{GPO}\xspace}

\newcommand{\Zooming}{\texttt{Zooming}\xspace}

\newcommand{\HOO}{\texttt{HOO}\xspace}

\let\inf\undefined
\DeclareMathOperator*{\inf}{\vphantom{\sup}inf}

\DeclareBoldMathCommand{\I}{I}
\DeclareBoldMathCommand{\e}{e}
\DeclareBoldMathCommand{\f}{f}
\DeclareBoldMathCommand{\g}{g}
\DeclareBoldMathCommand{\a}{a}
\DeclareBoldMathCommand{\b}{b}
\DeclareBoldMathCommand{\d}{d}
\DeclareBoldMathCommand{\m}{m}
\DeclareBoldMathCommand{\p}{p}
\DeclareBoldMathCommand{\q}{q}
\DeclareBoldMathCommand{\v}{v}
\DeclareBoldMathCommand{\V}{V}
\DeclareBoldMathCommand{\x}{x}
\DeclareBoldMathCommand{\t}{t}
\DeclareBoldMathCommand{\X}{X}
\DeclareBoldMathCommand{\Y}{Y}
\DeclareBoldMathCommand{\z}{z}
\DeclareBoldMathCommand{\Z}{Z}
\DeclareBoldMathCommand{\M}{M}
\DeclareBoldMathCommand{\n}{n}
\DeclareBoldMathCommand{\ssigma}{\sigma}
\DeclareBoldMathCommand{\SSigma}{\Sigma}
\DeclareBoldMathCommand{\OOmega}{\Omega}
\DeclareBoldMathCommand{\y}{y}
\DeclareBoldMathCommand{\U}{U}
\DeclareBoldMathCommand{\w}{w}
\DeclareBoldMathCommand{\W}{W}
\DeclareBoldMathCommand{\L}{L}

\DeclareBoldMathCommand{\s}{s}
\DeclareBoldMathCommand{\S}{S}
\DeclareBoldMathCommand{\A}{A}
\DeclareBoldMathCommand{\B}{B}
\DeclareBoldMathCommand{\C}{C}
\DeclareBoldMathCommand{\D}{D}
\DeclareBoldMathCommand{\E}{\mathbb{E}}
\DeclareBoldMathCommand{\G}{G}
\DeclareBoldMathCommand{\H}{H}
\DeclareBoldMathCommand{\P}{\mathbb{P}}
\DeclareBoldMathCommand{\Q}{Q}
\DeclareBoldMathCommand{\R}{R}
\DeclareBoldMathCommand{\X}{X}
\DeclareBoldMathCommand{\mmu}{\mu}
\DeclareBoldMathCommand{\ones}{1}
\DeclareBoldMathCommand{\zeros}{0}

\newcommand{\tree}{\mathcal{T}}

\newcommand{\depthOp}{\bot}
\newcommand{\partition}{\mathcal{P}}

\newcommand{\pulledArm}{x}

\newcommand{\timeHorizon}{n}

\newcommand{\learnerDist}{\p}
\newcommand{\pullsNumber}{T}

\newcommand{\uniRob}{\textsc{Robuni}}

\newcommand{\Pone}{\textsc{p1}}

\newcommand{\Exp}{\E}
\newcommand{\Pro}{\P}

\newcommand{\eventbob}{\xi}

\newcommand{\TODO}[1]{
\ifmmode
\text{\textcolor{red}{TODO: #1}}
\else
\textcolor{red}{TODO: #1}
\fi
}

\renewcommand{\epsilon}{\varepsilon}
\renewcommand{\hat}{\widehat}
\renewcommand{\tilde}{\widetilde}
\renewcommand{\bar}{\overline}

\newcommand{\Real}{\mathbb{R}}
\newcommand{\Integer}{\mathbb{N}}

\newcommand{\dom}{\mathcal X}

\newcommand{\lambertW}{W}

\newtheorem{assumption}{Assumption}

\newcommand\numberthis{\addtocounter{equation}{1}\tag{\theequation}}

\usepackage{hyperref}

\begin{document}
    
%

%







\twocolumn[

\aistatstitle{Derivative-Free \& Order-Robust Optimisation}




 \aistatsauthor{ Victor Gabillon,$^{1}$ \And Rasul Tutunov,$^{1}$\And Michal Valko,$^{2}$ \And  Haitham Bou Ammar$^{1}$}

\aistatsaddress{ Huawei R\&D UK$^{1}$ \And Inria Lille-Nord Europe$^{2}$} ]

\begin{abstract}
	In this paper, we formalise order-robust optimisation as an instance of online learning minimising simple regret, and propose \Vroom, a zero'th order optimisation algorithm capable of achieving vanishing regret in non-stationary environments, while recovering favorable rates under stochastic reward-generating processes. Our results are the first to target simple regret definitions in adversarial scenarios unveiling a challenge that has been rarely considered in prior work.
\end{abstract}

\section{Introduction}\label{s:intro}
Derivative-free optimisation is a discipline by which learners attempt to determine optimal solutions while only exploiting function value information~\citep{matyas1965random}. Such a setting is of great interest for applications in which it is either difficult to define, access or even compute first and/or second-order function information~\citep{nesterov2017random}. As such, derivative-free optimisation naturally addresses optimising over functions that are non-differentiable, non-continuous or even non-smooth.  
    
A variety of versatile zero-order methods have been developed under minimal smoothness assumptions~\citep{auer2007improved,kleinberg2008multi}. Though flexible, most algorithms in the literature are designed under specific assumptions on the process by which evaluation data is generated. \SOO~\citep{munos2011optimistic}, for instance, optimises sequentially over a deterministic function, while \StoSOO\citep{valko2013stochastic} optimises a sequence of noisy but stationary functions. No such algorithm, however, handles a sequence of non-stationary observations -- a setting commonly faced in a variety of real-world problems. Consequently, in a scenario in which the process generating the data is unknown a priori, what algorithm would a practitioner employ?

To illustrate the above concept, consider a lifelong learning  problem~\citep{thrun1995lifelong, ammar2014online,parisi2019continual} where a model is updated while interacting with a sequence of tasks. Here, the objective is to have a learner capable of performing well on average over all observed data. If the tasks are similar, learning online helps in solving novel tasks. However, when task differences are drastic, catastrophic forgetting occurs~\citep{french1999catastrophic,kirkpatrick2017overcoming} leading to situations where newly observed data hurts performance on earlier problems. In fact, it has been reported that the order by which tasks are streamed dramatically affects average performance. It is for this reason that recent research in lifelong learning has focused on building \emph{order-robust} approaches~\citep{yoon2019oracle} that we formalise in this work as an instance of online learning with simple regret considerations.

Precisely, we formalise the above problem by optimising over elements $x$ in a continuous set $\dom$. $n$ tasks are streamed sequentially allowing the learner to attempt a sequence $x_1,\ldots,x_n$ across rounds. At round $t$, the learner observes a reward $f_t(x_t)$ corresponding to the performance of parameter $x_t$ on task $t$ as represented by the mapping $f_t$. After $n$ rounds, the agent recommends a parameter $x(n)$ with the objective of maximizing its average reward over all observed tasks, i.e., $\frac{1}{n}\sum_{t=1}^n f_t(x(n))$. 


%
%
 
Contrary to other methods in the literature, we believe that minimal assumptions on the order by which $f_1, \ldots, f_n$ are observed have to be invoked to ensure order-robustness. Furthermore, our algorithm should also behave near optimally as if an a priori knowledge of such an order (e.g., stochastic observations) was explicitly provided. Interestingly, this motivation unveils a novel problem which we refer to as \textbf{the best of both worlds (BOB) challenge}. Here, we aim to design one simple algorithm that is unaware of the nature of the reward generating process but can acquire near-optimal regret guarantees in both stochastic and adversarial non-stationary settings. In this paper, we take the first step to resolving the aforementioned challenge by proposing \Vroom a novel algorithm that optimises over~$f$ at different levels of discretisation of the input space $\dom$. \Vroom makes use of the standard importance-weighted estimates used in non-stochastic literature for efficient exploration. We realise, however, that the direct application of these techniques to our setting suffers from two major drawbacks related to variance explosion when observation probabilities diminish with discretisation widths, and estimate unreliability due to variance disparities. Providing solutions to each of these above problem, our contributions are summarised as: 1) formally introducing simple regret minimisation in non-stochastic and order-robust optimisation, 2) analysing a uniform exploration algorithm and demonstrating state-of-the-art bounds in non-stochastic settings, and 3) introducing \Vroom as a solution to order-robustness proving vanishing regrets in non-stochastic settings and $\tcO\left(n^{-\frac{1}{d+3}}\right)$ in the stochastic case.

\section{Problem Formulation and Analysis Tools}\label{sec:ProFo}
    
    In this section, we detail our problem formulation, its novelty, the associated challenge, and our contributions.
    
    In budgeted  optimisation, a learner optimises a function $f: \dom  \rightarrow\Real$ having access to a number of evaluations limited by $\timeHorizon$. This setting also includes the case   $\dom\subset \Real^D$.
We consider a general case where $f$ is decomposable as,
\[
f=\frac{1}{n}\sum_{t=1}^n f_t.
\]
It is clear that $f$ depends on $n$. However, since $n$ is a fixed input parameter of the problem, we drop such dependency in our notation for ease of exposition.
At each round $t \in \{1, \dots, \timeHorizon\}$, the learner chooses an element $x
    _t\in\dom$  and observes a real number $y_t$, where $y_t=f_t(x_t)$ quantifying its reward. As we are concerned with order-robustness, we distinguish two feedback settings with respect to the process by which $f_t$'s are interconnected:
    \begin{description}
    \item[Stochastic feedback] In stochastic feedback, function evaluations are
    perturbed by a noise in the range $b\in\Real_+$\footnote{Alternatively, we can turn the boundedness assumption into a sub-Gaussianity assumption equipped with a variance parameter equivalent to our range $b$.}:
        Precisely, at any round, we have $f_t= \bar f +\epsilon_t$ with $\epsilon_t$ being a random variable that is identically and independently distributed (i.i.d.) over rounds. Further, we consider the case when $\bar f$ is a function that is independent of $t$ and $n$, and where: 
    \begin{equation} \label{eq:store}
    \Exp \left[ \epsilon_t \right] = 0\quad \text{ and } \quad |\epsilon_t|\leq b.
    \end{equation}
    \item[Non-stochastic feedback] To consider non-stationary and non-stochastic data, we  minimally assume:
    \begin{equation} \label{black}
     \quad |f_{t'}(x)-f_t(x)|\leq b \text{ for all }  t,t' \text{ and } x\in\dom.
    \end{equation}
    \end{description}
Given these feedback laws, the learner's objective is to return an element $x(n)$ in $\dom$ with the largest possible value $f\left(x\right)$ after the $\timeHorizon$ evaluations. To that end, we allow the learner to utilise internal randomisation, i.e., sample $x(n)$ from a distribution $\nu_n$ of its choice, $x(n)\sim \nu_n$. 

Since we consider two feedback laws (i.e., stochastic and non-stochastic), we quantify the agent's performance using two notions of simple regrets. In the first, we consider regret as a random variable induced by $\epsilon_{1}, \dots, \epsilon_{n}$ and bound its expectation over the random sequence $f_{1}, \dots, f_{n}$: 
    \[
    \begin{aligned}\label{eq:regrSto}
\Exp_f \left[r_\timeHorizon\right]
 &\triangleq 
\Exp_{f_1,\ldots,f_n}\left[
 \sup_{x\in\mathcal X} f\left(x\right)- \Exp_{x(n)}\left[f\left(x\right)\right] 
 \right]\\
  &=\sup_{x\in\mathcal{X}}\bar{f}(x) - \mathbb{E}_{f_1,\ldots,f_n}\left[\mathbb{E}_{x(n)}\left[\bar{f}(x(n)) + \sum_{t=1}^n\epsilon_t\right]\right]  \\
        &=\sup_{x\in\mathcal{X}}\bar{f}(x) - \mathbb{E}_{x(n)}\left[\bar{f}(x(n))\right]- \mathbb{E}_{f_1,\ldots,f_n}\left[ \sum_{t=1}^n\epsilon_t\right] \\
&=
 \sup_{x\in\mathcal X} \bar f\left(x\right)- \Exp_{x(n)}\left[ \bar f\left(x\right)\right],
    \end{aligned}
    \]
where the expectation with respect to $\nu_{n}$. When it comes to the non-stochastic setting, the situation is simpler where for a given sequence of function observations, we define:
        \begin{equation}\label{eq:regrAdv}
r_\timeHorizon
 \triangleq 
 \sup_{x\in\mathcal X} f\left(x\right)- \Exp_{x(n)}\left[f\left(x\right)\right] \,,
    \end{equation}
    We further consider the case when evaluation is costly. Therefore, we minimise~ $r_\timeHorizon$ as a function of $\timeHorizon$ assuming that for any given sequence $f_1,\ldots,f_n$, there exists at least one point $x^\star\in \mathcal X$  such that
    $f(x^\star) = \sup_{x\in\mathcal X} f(x)$.
    %


    
Before commencing with our solution, it is worth noting that optimising simple regret with non-stochastic data generating processes has not been studied as a stand-alone problem in literature so far\footnote{Section~\ref{sec:relatedwork} extensively reviews the long history of existing results for stochastic and deterministic feedback laws.}. It is viewed by some authors as an ill-defined problem \cite[Chapter 3]{hazan2016introduction} as the objective varies at each round $t$. Moreover, if the simple regret is formulated as in Equation~\ref{eq:regrAdv}, one can, in some cases, derive bounds for such a quantity from the analysis of cumulative regret, $\sup_{x\in\dom}\frac{1}{n}\sum_{t=1}^n f_t(x)-\frac{1}{n}\sum_{t=1}^n f_t(x_t)$ -- a notion extensively studied in~\citep{auer2002using,zinkevich2003online,bubeck2017kernel}. 
In the stochastic setting or when $f_t=f_1$ for $t\in [n]$, 
obtaining an upper bound, $R_n$, on the cumulative regret leads to an upper bound $r_n=R_n/n$ on the simple regret as noted in~\cite{hazan2016introduction,bubeck2011x}.
It is worth noting that though a bound can be attained, these two objectives are not equivalent. Precisely, a bound obtained in the cumulative regret case is often sub-optimal from a simple regret point of view~\citep{Bubeck09PE}. Furthermore, contrary to simple-regret algorithms, cumulative-regret learners find it challenging to adapt function smoothness without extra information on $f$~\citep{locatelli2018adaptivity}. In fact, it is intuitive to realise that minimising cumulative regret aims at accumulating rewards (see the term $\frac{1}{n}\sum_{t=1}^n f_t(x_t)$), as opposed to identifying the optimum (the term $\frac{1}{n}\sum_{t=1}^n f_t(x(n))$); a property dictated through simple regret considerations. Finally, note that to the best of our knowledge no upper bound on the cumulative regret exists in non-stochastic settings under minimal assumptions on $f$ used in this paper and that the connection between cumulative and simple regret is unclear in the non-stochastic setting.

\subsection{Mathematical Tools}    
During the remainder of this paper, we will make use of mathematical tools that we briefly survey in this section. Firstly, we describe partitioning assumptions facilitating our search for an optimal solution of our optimisation problem, and then detail tree-based learners that we build on in developing \Vroom.   

\subsubsection{Partitioning \& Near-Optimality Dimension}
    During our exploration for an optimum, we discretise the search space into cells (nodes) allowing us to consider tree-like learners. To do so, we follow a hierarchical partitioning $\mathcal{P} = \{\{\mathcal{P}_{h,i} \}^{I_h}_{i=1}\}^{\infty}_{h=0}$ previously introduced in~\citep{munos2011optimistic,valko2013stochastic,grill2015black-box}. For any
    depth $h\geq 0$ (which we think of as a tree representation), the set
    $\{\partition_{h,i}\}_{1\leq i \leq I_h}$
    of
    \textit{cells} (or nodes)
    forms a partition of
    $\dom$, where
    $I_h$
    is the number of cells at depth
    $h$. At depth $0$, the root of the tree, there is a single cell
    $\partition_{0,1}=\dom$. A cell $\partition_{h,i}$ of depth $h$ is split into children sub-cells $\{\partition_{h+1,j}\}_j$
    of depth $h+1$. The objective of many algorithms is to explore the value of~$f$ in the cells of the partition and to identify at the deepest possible depth a cell containing a global maximum. 
For simplicity and without loss of generality we assume all cells have $K$ children sub-cells.



    Given a global maximum $x^\star$ of $f$,
    $i^\star_{h}$
    denotes the index of the unique cell of depth
    $h$
    containing
    $x^\star$
    , i.e., such that
    $x^\star\in \partition_{h,i^\star_{h}}$.
    We follow the work of~\cite{Grill15BB} and state a \emph{single} assumption on
    both the partitioning
    $\partition$
    and the function~$f$.
    \begin{assumption}\label{as:smooth}
        For any global optimum $x^\star$, there exists $\nu>0$ and $\rho\in (0,1)$, where the values of $\nu$ and $\rho$ depend on  $x^\star$, such that $\forall h \in \Integer $,
        $\forall x \in \mathcal P_{h,i^\star_h},  f(x)    \geq f(x^\star) - \nu\rho^h.$
    \end{assumption}
    The notion of a near-optimality dimension~$d$ aims at capturing the smoothness of the function and characterises the complexity of the optimisation task. 
    We adopt the definition of near-optimality dimension given recently by~\citet{grill2015black-box} that unlike \cite{bubeck2011x}, \cite{valko2013stochastic}, \cite{munos2011optimistic}, and \cite{azar2014online},
    avoids topological notions and does not artificially attempt to separate the difficulty of the optimisation from the partitioning. For each depth $h$, it simply counts the number of near-optimal cells $\mathcal{N}_h$, i.e., those whose value is close to~$f(x^\star)$, and determines how this number evolves with the depth $h$.
    The smaller the depth~$d$, the more accurate is the optimisation. 
    \begin{definition}\label{def:neardim}
        For any $\nu > 0$, $C>1$, and $\rho \in (0,1)$, the \textbf{near-optimality dimension}\footnote{\cite{Grill15BB} define  $d(\nu,C,\rho)$ with the constant 2 instead of 3.  3 eases the exposition of our results.} $d(\nu,C,\rho)$ of~$f$ with respect to the partitioning
        $\mathcal P$,
        is 
        %
        \[
        d(\nu,C,\rho) \triangleq    \inf \left\{d'\in\Real^+:  \forall h \geq 0, ~\mathcal N_h(3\nu\rho^h) \leq     C\rho^{-d'h} \right\}\CommaBin
        \]
        where
        $\mathcal N_h(\epsilon)$ is the number of cells
        $\mathcal P_{h,i}$    of depth $h$ such that
        $\sup_{x\in \mathcal P_{h,i}} f(x) \geq    f(x^\star) - \epsilon$.
    \end{definition}
    %
    %
    %
    By construction we have $d\leq \log(K)/\log(1/\rho)$. In general $d\ll \log(K) \log(1/\rho)$  as having $d=0$ is the most common case in practice~\citep{valko2013stochastic}.
    
    \subsubsection{Tree-Based Learners}
    Tree-based exploration or a tree search algorithm is an approach that has been widely applied to optimisation as well as bandits or planning problems \citep{kocsis2006bandit,coquelin2007bandit,hren2008optimistic}; see \cite{munos2014from} for a survey.
    %



        First we define the sampling of an element $x$ in a cell $\partition_{h,i}$ \emph{with respect to $\partition$}, denoted  $x\sim U_\partition(\partition_{h,i})$ as follows: Starting from a cell $c_1=\partition_{h,i}$, we descend the partition until depth $n$ by choosing at depth $h'$ (with $h \le h'< n$) a sub-cell $c_{h'+1}$ of $c_{h'}$ chosen uniformly at random among the $K$ children cells of $c_{h'}$. Once at depth $n$ in $\partition_{n,i}$, we pick an element $x$ uniformly at random in $\partition_{n,i}$. \footnote{Assuming that each parent cell has $K$ children,  sampling from $U_\partition(\partition_{h,i})$  is just a uniform sampling from the descendants of $\mathcal{P}_{h,i}$ at depth $n$. If we assume that each cell can have different number of children, then $U_\partition(\partition_{h,i})$ follows the topology of $\partition$.}

        At each round $t$, the learner selects an element $x_t\in \dom$. 
        First the learner selects a cell $\partition_{h_t,i_t}$ according to the distribution 
        $\learnerDist_t$ on $\partition$ that associates to each cell $\partition_{h,i}$ the probability
        $\learnerDist_{h,i,t}=\Pro\left(\partition_{h_t,i_t}=\partition_{h,i}\right)$ of being the selected cell $\partition_{h_t,i_t}$ at time $t$.
        We have $\sum_{\partition_{h,i}\in\partition} \learnerDist_{h,i,t}=1$ for any given $t$.
        Then, the learner samples  an element $x_t$ in $\partition_{h_t,i_t}$ with respect to $\partition$,  $x_t\sim U_\partition(\partition_{h_t,i_t})$, and asks for its evaluation.

    We denote the value  $f_{h,i} \triangleq \Exp_{x\sim U_\partition(\partition_{h,i})}\left[      f(x)\right]$, $f_{h,i,t} \triangleq \Exp_{x\sim U_\partition(\partition_{h,i})}\left[      f_t(x)\right]$ and, in the stochastic feedback case, $\bar f_{h,i}=  \Exp_{x\sim U_\partition(\partition_{h,i})}\left[    \bar f(x)\right]$.
    We use
    $\pullsNumber_{h,i}(t)= \sum_{s=1}^{t-1} 1_{x_s\in\partition_{h,i}}$ to denote the total number of evaluations that have been allocated by the learner between round $1$ and the beginning of round $t$ to the cell $\partition_{h,i}$. 
    %
        For the stochastic noisy case, we also define the estimated value of the cell $\mathcal{P}_{h,i}\in\mathcal{T}$ as follows:
     given the $\pullsNumber_{h,i}(t)$ evaluations $y_1,\ldots,y_{\pullsNumber_{h,i}(t)},$ we have 
     \[
     \hat f_{h,i}(t) \triangleq \frac{1}{\pullsNumber_{h,i}(t)} \sum_{s=1}^{\pullsNumber_{h,i}(t)} y_s,
     \]
    the  empirical  average  of  rewards  obtained  at this cell.

        Similarly, for the non-stochastic case, we define     $\tilde f_{h,i}(t)$ that estimates $ f_{h,i,t}$ for cell $\partition_{h,i}$ at time $t$. This estimates uses the function values $f_t(x_t)$ if collected from sampling directly cell $\partition_{h,i}$ as $x_t\sim U_\partition(\partition_{h,i})$ which corresponds to $h_t=h$ and $i_t=i$. In addition, the estimate     $\tilde f_{h,i}(t)$ also takes into account $f_t(x_t)$ if both $x_t\in\partition_{h,i}$ and $h_t\leq h$. This addition improves the accuracy of our estimate while forcing  $h_t\leq h$
 insures that $f_t(x_t)$ is an unbiased estimate of the quantity of interest $ f_{h,i,t}$ as proven below.
    %
    Having a sample  $x_t\sim U_\partition(\partition_{h_t,i_t})$ with  $\partition_{h_t,i_t}\sim \learnerDist_t$, possibly $\partition_{h_t,i_t}\neq\partition_{h,i}$,  
 and an observation $y_t=f_t(x_t)$, we have 
    \begin{equation}\label{es:unbiased}
    \tilde f_{h,i}(t) 
    \triangleq
    \frac{ y_t 1_{x_t\in \partition_{h,i}} 1_{h\geq h_t} 
    }{
     \Pro(\pulledArm_t\in\partition_{h,i} \cap h\geq h_t)
    }.
    \end{equation}
        \[
        \begin{aligned}
        \text{ Note that~~~~~~~~~~~~}\mathbb{E}_{\mathcal{P}_{h_t,i_t}\sim\boldsymbol{p}_t}[\Exp_{x_t\sim U_{\mathcal{P}}(\mathcal{P}_{h_t,i_t})}[\tilde f_{h,i}(t)]]\\
        =
                \Exp_{x_t\sim U_{\mathcal{P}}(\mathcal{P}_{h_t,i_t})}[
             y_t  | x_t\in \partition_{h,i} \text{ and } h\geq h_t 
                ]\\
                        \stackrel{\textbf{(a)}}{=}
                \Exp_{x_t\sim U_{\mathcal{P}}(\mathcal{P}_{h,i})}[
             y_t  
                ]
         =  f_{h,i,t}
        .\end{aligned}
        \]
        where \textbf{(a)} is by definition of $U_{\mathcal{P}}(\mathcal{P}_{h_t,i_t})$.
We define
            $ \tilde F_{h,i}(t) \triangleq  \sum_{s=1}^{t} \tilde f_{h,i}(t)$,
    the  sum  of  rewards  obtained  at this cell.
We define     $ F_{h,i}(t) \triangleq  \sum_{s=1}^{t}    f_{h,i,s}$.
 Finally,
    let $[a:c]=\{a,a+1,\ldots,c\}$ with
    $a,c\in\Integer$, $a\leq c$, and $[a]=[1:a]$. 
    $\log_{d}$ denotes the logarithm in base $d\in\Real$. Without a subscript, $\log$ is the natural logarithm in base $e$.

      However, this method does not fit well the cases where we need to sample
      a large number of cells with a limited amount of pulls such as low noise settings, deterministic feedback and $d=0$ for which \StoSOO has no theoretical guarantees.
     In \StroquOOL a separate cross-validation phase is allocating  $\tcO(n)$ extra samples to the best cells that are recommended at the end of the initial exploration phase. However, when dealing with non-stochastic data there are no guarantees that the data collected in the two phases are related therefore introducing a bias that happens to be hard to control and which introduces the undesired parameter of the length of the exploration phase.

\section{VROOM: Simple Algorithm for Order-Robust Optimisation}
\label{allourcontributions}
This section details our contributions to addressing order-robustness. On a high level, we split the exposition in three parts. First, we provide a robust version of uniform exploration that sets state-of-the-art regret guarantees for non-stochastic settings. While these guarantees are believed to be unimprovable, uniform exploration is known to perform sub-optimally in stochastic scenarios. As such, we revert-back to the BOB challenge discussing achievable regret rates before presenting \Vroom.   

Before diving into details of our proposed method, it is instructive to recap the challenges faced when considering two feedback laws. Targeting only stochastic feedbacks, it is well known that \StroquOOL and \GPO, achieve state-of-the-art regret bounds. Unfortunately, the direct application of these methods to an adversarial setting is challenging due to the potential blunder that can be caused by feeding uninformative rewards for a deterministic learner as pointed in~\citet[Section~3]{Bubeck12RA}. Therefore, it is essential for an efficient learner to employ internal randomisation that defines a positive probability $\mathbb{P}\left(x_{t} \in \mathcal{P}_{h,i}\right)$ for each cell during its exploration quest. Given positive probabilities, we can now target an estimator for $f(x)$ to perform meaningful updates. Clearly, the simple usage of empirical averaged rewards $ \hat f_{h,i}(t) $ in cell $\partition_{h,i}$ is easily biased by an adversary. Fetching an unbiased estimate, we realise that $ \tilde f_{h,i}(t)$ is a meaningful alternative. Though viable, $\tilde f_{h,i}(t)$ can possess high variance especially if $\learnerDist_{h,i,t}$ is small (scaling with $1/\Pro(\pulledArm_t\in\partition_{h,i})$). Two sources contribute to these high variance occurrences: 1) long uniform exploration, and 2) $K^h$ increase in the number of cells with depth (leading to variances of $K^h$ magnitude). Realising these problems, we present our first challenge that we tackle in this paper as: \\
\underline{\textbf{Challenge I:}} How to control potentially large estimator variances (especially in the stochastic setting)?

Apart from variance control, we face another interesting problem related to the optimum recommendation, $x(n)$, made by the learner after $n$ rounds of interaction. If we are to recommend the best cell as that with the highest estimate $\sum_{t=1}^n\tilde f_{h,i}(t)$, we might end-up comparing estimates with widely different confidence intervals\footnote{Please note that this is due to the dependence on the number of pulls allocated to $\partition_{h,i}$, as well as on the variance of the estimates.}. At first sight, one can attempt to follow the approaches proposed by others in the literature to tackle this issue. In \StoSOO, for instance, $x(n)$ is chosen among the cells that have been pulled in an order of $\tcO(n)$. Though appealing, this method does not fit-well the cases where we need to sample a large number of cells with a limited number of pulls such as in the low noise, deterministic feedback and/or $d=0$ settings\footnote{In such cases \StoSOO lacks any theoretical guarantees.}. In \StroquOOL, on the other hand, a separate cross-validation phase allocates $\tcO(n)$ extra samples to the best cells recommended at the end of an initial exploration phase. Nonetheless, when dealing with non-stochastic reward-generating processes, there are no guarantees on the relationship between collected data in two successive phases. Hence, following such a recommendation introduces a (hard-to-control) bias typically leading to additional hyper-parameters measuring exploration lengths. Observing optimum recommendation difficulties arising from considering two feedback laws, our second challenge can be stated as:\\
\underline{\textbf{Challenge II:}} How to recommend an optimum $x(n)$ capable of operating successfully in both feedback settings?

The remainder of this section provides solutions to each of the above challenges ultimately proposing \Vroom as a simple yet effective algorithm for order-robust optimisation. 



\subsection{Uniform Allocation Baselines}
\label{unifrom}
In this section, we derive achievable baseline simple regret rates in non-stochastic scenarios. We note that such a problem has not yet been targeted by current literature. To do so, we consider a uniform exploration strategy allowing us to achieve initial results addressing \textbf{Challenge~II}\footnote{Note that as the above exposition considers no stochasticity. As such, answers to \textbf{Challenge I} are considered in later sections when attempting to determine a best of both worlds algorithm.}. We specifically discuss two optimum recommendation techniques: 1) cross-validation, and 2) lower confidence bounds (LCBs). We report how existing cross-validation techniques can be used to obtain regret rates in a stochastic case where the learner is unaware of smoothness parameters and discuss corresponding limitations in non-stochastic settings. We then demonstrate that LCB allows building a robust version of uniform allocation~\uniRob{} for non-stochastic environments\footnote{We report the complete proofs in Appendix~\ref{app:uni}.}.

\paragraph{Stochastic feedback}\label{sec:uniSto}
To determine valid regret rates, we distinguish two scenarios depending on the knowledge of smoothness parameters. First, uniform strategy exploits $(\nu,\rho)$, while, second, the learner is oblivious to $(\nu,\rho)$.

\emph{$\circ$ With knowledge of $(\nu,\rho)$:}
At depth $h$ a uniform algorithm can explore all $K^h$ cells $\left\lfloor n/K^h \right\rfloor$ times. Such a strategy recommends a valid parameter $x$ that attains the highest observed $\hat f(x)$.
At depth $h$, errors are bounded by $\nu\rho^h$, and the estimation error is given by $\cO\left(\sqrt{K^h/n}\right)$. Optimising over $h$ for the sum of these two errors, we can state that by setting $H=\left\lfloor\log_{K/\rho^2}(n)\right\rfloor$, $r_n = \tcO \left(\log(1/\delta)/n\right)^{\frac{1}{\frac{\log  K}{\log  1/\rho}+2}}$ with probability at least $1 - \delta$.

\emph{$\circ$ Without knowledge of $(\nu,\rho)$:}
So far, we derived a bound where the optimal choice of $H$ is dependent on smoothness parameters $(\nu,\rho)$. When not knowing $(\nu,\rho)$, our strategy uses a budget of $n/2$ rounds to explore all depths $h\in [0:\lfloor \log_{K}(n) \rfloor]$. A depth $h$ is explored uniformly with a budget of 
$n/(2\lfloor \log_{K}(n) \rfloor )$. We define $\lfloor \log_{K}(n) \rfloor$ candidates $x_h$ with the highest observed $\hat f(x_h)$ among the cells of depth $h$. Now, the final recommendation corresponds to a choice between these $\lfloor \log_{K}(n) \rfloor$ candidates. However, each has been pulled   $T_{x_h}(n/2) = \frac{n/(2\lfloor \log_{K}(n) \rfloor )}{K^h} $ number of times and as such, arrives with different confidence estimates.  We can implement a \emph{cross validation} step, as used in~\cite{bartlett2019simple}, which only requires $n$ to make use of the remaining $n/2$ rounds. Each  candidates is sampled additionally $n/(2\lfloor \log_{K}(n) \rfloor )$ 
This leads us to obtain  $r_n = \tcO \left(
\left(\frac{K}{n\rho^2}\right)^{\frac{1}{\frac{\log  K}{\log  1/\rho}+2}}
\right)$.
With this strategy, wecrecover the same results as if smoothness parameters were provided up to a logarithmic factor. An alternative to cross validation with same theoretical guaranties is that, after a uniform allocation on all cells at a depth smaller than $\lfloor \log_{K}(n) \rfloor$, to recommend among all cells these with largest \emph{lower confidence bound} $\hat f_{h,i}(n)-b\sqrt{\frac{\log(n^2/\delta)}{T_{h,i}(n)}}$. 
This allows to compare candidate cells at different depths by taking into account
the uncertainty  $b\sqrt{\frac{\log(n^2/\delta)}{T_{h,i}(n)}}$ around their estimated averages.
Though this approach requires the knowledge of $b$ (the range of $\epsilon_{t}$), it will come handy in the non-stochastic setting detailed next. 



\paragraph{Non-stochastic feedback}

\begin{figure}
    \centering
    \framebox{
        ~    \begin{minipage}{.42\textwidth}
            
            \textbf{Parameters:} $\partition 
            =
            \{\partition_{h,i}\}$,  $b,n,f_{\max}$\\[.05cm]
     
     Set $\delta=\frac{4b}{f_{\max} \sqrt{n}}$.

            \textbf{For} $t=1,\ldots,n$ \textit{\textbf{\textcolor{gray}{\hfill $\blacktriangleleft$  Exploration $\blacktriangleright$}}}

        ~~~~\, Evaluate a point $x_t$ sampled from $ U_\partition(\partition_{0,1})$.   \\[.15cm]  

            \textbf{Output} $x(n)\sim \mathcal{U}(\partition_{h(n),i(n)})$\\ where  $ (h(n),i(n))  \gets  \argmax\limits_{h,i} \tilde F_{h,i}(n) - B^{adv}_{h}(n)$

        \end{minipage}
    }
    \caption{The \uniRob{} algorithm}\label{fig:uniAD2}
\end{figure}

As discussed above, in the non-stochastic setting we use a uniform allocation combined with a recommendation based lower confidence estimate of the value of cell $\partition_{h,i}$ as 
$\tilde F_{h,i}(n) -
B^{adv}_{h}(n)$ where $B^{adv}_{h}(n) \triangleq
\sqrt{ 2n f^2_{max}K^h   \log(\timeHorizon^2/\delta)  } - \frac{f^2_{max}}{3}K^h\log(\timeHorizon^2/\delta)$.
We name such an algorithm \uniRob{} and detail its pseudo-code in Figure~\ref{fig:uniAD2}.
\uniRob{} is required knowledge of $b$ (See Equation~\ref{black}), and
$f_{max}$  that  upper bounds the maximal value of the functions $f_1,\ldots,f_n$, i.e.,  
$|f_t(x)|\leq f_{max}$ for all $x\in\dom$ and $t\in\{1,\ldots,n\}$.


We are now ready to present the simple regret bounds attained by \uniRob{} in the following theorem:
\begin{restatable}[\textcolor{titleTh}{Upper bounds for \uniRob{}}]{theorem}{order}    \label{th:unilcbad}

Consider any sequence of functions $f_1,\ldots,f_n$ such that $|f_t(x)|\leq f_{max}$ for all $x\in\dom$ and $t\in [n]$. Let $f=\frac{1}{n}\sum_{t=1}^n f_t$, and $x^\star$  be one of the global optima of $f$ with associated $(\nu,\rho)$. Then 
        after $\timeHorizon$ rounds,     the simple regret of {\uniRob} is bounded as:
        \[         
        \Exp \left[r_\timeHorizon\right]
        =
        \cO \left(\log(n/\delta)\left(\frac{K}{n\rho^2}\right)^{\frac{1}{\frac{\log  K}{\log 1/\rho}+2}}\right) 
        \]
\end{restatable}
%
%
%
%
The above result demonstrates that using \uniRob{} uniform exploration strategies can be made order-robust retaining same regret guarantees in the non-stochastic setting as those obtained in the stochastic case. However, we conjecture that this is not true for most learners, where we believe that any algorithm can only obtain, at best, the same regret rates as \uniRob{} in non-stochastic cases. This is not unlike best-arm identification problems( when $\dom$ is reduced to $\dom=[K]$), where the authors in~\citep{abbasi-yadkori2018best} showed unimprovable regret rates to those obtained by uniform strategies.




%

\subsection{Achievable Rates for BOB}
\begin{table}
{\center    
        \begin{tabular}{|c|c|c|c|c|}
        \hline
             &
             $b=0$ &    stochastic ($b>0$) & non-sto \\ \hline
                                    \Vroom &          \textbf{open} &
                                    $\frac{1}{n}^{\max\big(\frac{1}{d+3}, \frac{1}{\frac{\log  K}{\log \frac{ 1}{\rho}}+2} \big)}$& $\frac{1}{n}^{\frac{1}{\frac{\log  K}{\log \frac{ 1}{\rho}}+2}}$\\ \hline
            \StroquOOL    & $\left(\frac{1}{n}\right)^{\frac{1}{d}}$ &        $\left(\frac{1}{n}\right)^{\frac{1}{d+2}}$& \xmark\\  \hline
            \SequOOL   &         $\left(\frac{1}{n}\right)^{\frac{1}{d}}$ & \xmark & \xmark \\ \hline
            Uniform(s) & $\frac{1}{n}^{\frac{\log \frac{ 1}{\rho}}{\log  K}}$     & \multicolumn{2}{c|}{ $1/n^{\frac{1}{\frac{\log  K}{\log  1/\rho}+2}}$ }\\ \hline
        \end{tabular}
        \caption{$\tcO$ rates of SOTA in  deterministic, stochastic and non-stochastic settings. \xmark~  denotes a non-vanishing regret. Though \Vroom stochastic bounds can be applied  when $b=0$, we leave a better bound open direction of future research.}\label{tablle}
        }
    \end{table}

Though the uniform exploration algorithm discussed above achieves order-robustness in non-stochastic settings, it can become highly sub-optimal for stochastic scenarios. In fact, it is well known that for stochastic data generating processes, \StroquOOL and \GPO obtain a state-of-the-art simple regret of the order $\tcO\left(\left(\frac{1}{n}\right)^{\frac{1}{d+2}}\right)$. Yet, as detailed in Section~\ref{sec:ProFo}, one can design a sequence of functions (i.e., non-stochastic scenario) $f_1,\ldots,f_n$ with any associated parameter $d,\nu,\rho$ such that simple regret of \StroquOOL, for instance, is lower bounded by a constant for any $n$. 

Given the lack of  algorithm performing well in both scenarios, we next attempt to design a learner that is unaware of the nature of the reward-generating process but  simultaneously achieves near-optimal simple regret bounds, i.e., 
\begin{align*}
    \mathbb{E}[r_{n}] & =\tcO\left(\left(\frac{1}{n}\right)^{\frac{1}{d+2}}\right) \ \ \text{(stochastic feedback)} \\
    \mathbb{E}[r_{n}] & =\tcO\left(\left(\frac{1}{n}\right)^{\frac{1}{ \frac{\log(K)}{\log(1/\rho)}+2  }}\right) \ \ \text{(non-stochastic feedback)}
\end{align*}
%
 
%
 %
 \label{sec:rateExp}
\paragraph{Rates of Optimality:} To understand the optimality statements that can be considered when tackling both scenarios, we draw upon results from best-arm identification problems, i.e., when $\dom=[K]$. There, \cite{abbasi-yadkori2018best} showed that obtaining optimal rates in stochastic and non-stochastic cases simultaneously is impossible. We conjecture that this result carries to our setting, where we believe simultaneous optimal rates are also not achievable. 

This, consequently, poses the question of what type of optimal rates can an algorithm obtain in stochastic feedback settings, while still guaranteeing vanishing regrets in non-stochastic cases. A formal lower bound guarantee of optimality is beyond the scope of this paper, and is left as an open question for future research. We do, however, demonstrate \Vroom to be the first algorithm acquiring vanishing regrets in non-stochastic scenarios, while still achieving favourable rates compared to state-of-the-art stochastic algorithms, i.e., $\mathbb{E}[r_{n}] =\tcO\left(n^{-\frac{1}{d+3}}\right)$.







\subsection{Robust optimisation}
\label{stochastic}
%
%
%
%
%

In this section, we present a new learner and analyse its theoretical performance against any i.i.d. stochastic problem or any non-stochastic environment.


\begin{figure}
    \vspace{-0.4cm}    \centering
    \framebox{
        ~    \begin{minipage}{.45\textwidth}
            
    \textbf{Parameters:} $\partition 
            =
            \{\partition_{h,i}\}$,  $b,n,f_{\max}$\\
     Set $\delta=\frac{4b}{f_{\max} \sqrt{n}}$.\\[.01cm]

    \textbf{For} $t=1,\ldots,n$ \textit{\textbf{\textcolor{gray}{\hfill $\blacktriangleleft$  Exploration $\blacktriangleright$}}}
    
                ~~~~\,    For each depth $h\in\left[\left\lfloor \log_K (n) \right\rfloor\right]$, rank\footnote{ Equalities between cells
            or comparisons with cells that have not been pulled yet are broken arbitrarily.} the cells  by decreasing order of $f^-_{h,i}(t-1) $: Rank cell $\partition_{h,i}$ as $\hat{\langle i \rangle}_{h,t}$. 
     
            


            ~~~~\,$x_t\sim \mathcal{U}_{\partition}(\partition_{h_t,i_t})$ where  
            $\partition_{h_t,i_t}$ is sampled so that
                 for any $h\in\left[\left\lfloor \log_K (n) \right\rfloor\right]$ and any    $i\in[K^h]$, 
            \[
            \learnerDist_{h,i,t}
            \triangleq
            \Pro\left(\partition_{h_t,i_t}=\partition_{h,i}\right) 
             \triangleq
             \displaystyle  \frac{ 1 }{~h\hat{\langle i \rangle}^{\vphantom{X}}_{h,t} \bar \log_K(n) ~}
             \]
        and where 
        $    \bar \log_K(n)=
            \sum^{\left\lfloor\log_K(n)\right\rfloor}_{h=1}\sum^{K^h}_{i=1}\frac{1}{hi}$.
            \\[.2cm]   
    
            \textbf{Output} $x(n)\sim \mathcal{U}_{\partition}(\partition_{h(n),i(n)})$ where  $ (h(n),i(n)) \gets  \argmax\limits_{(h,i) } \tilde F_{h,i}(n)- B_{h,i}(n)$

            
        \end{minipage}
    }
    \caption{ The  \Vroom algorithm}\label{fig:vroom}
\end{figure}

We title the algorithm \Vroom and detail it in Figure~\ref{fig:vroom}. Intuitively,
\Vroom first selects a depth $h$ with a probability inversely proportional $h$. Given its depth selection, \Vroom queries the best estimated cell with ``probability'' one, the second-best estimated cell with a ``probability'' of one half, and so on until pulling the worst-estimated cell with a ``probability'' $\frac{1}{K^h}$. To guarantee valid probabilities, we need a normalisation factor. As it is sufficient to sample depths $h\in[\left\lfloor \log_{K}(n)\right\rfloor]$, the normalising constant can be bounded as:  
\[
\begin{aligned}
\bar \log_K(n) &= \sum^{\left\lfloor\log_K(n)\right\rfloor}_{h=1}\sum^{K^h}_{i=1}\frac{1}{hi}
\leq \sum^{\left\lfloor\log_K(n)\right\rfloor}_{h=1}\frac{1}{h}(\log(K^h)+1) \\
&\leq 2\sum^{\left\lfloor\log_K(n)\right\rfloor}_{h=1}\log(K)\leq 2\log(n). 
\end{aligned}
\]

At round $t$, the estimate used in \Vroom to rank the cell during exploration is given by $\hat f^-_{h,i}(t-1)\triangleq\hat f_{h,i}(t-1)-B^{iid}_{h}(t-1) $ for cell $\partition_{h,i}$,
    where 
    $B^{iid}_{h}(t-1) =\sqrt{\frac{\log(4n^3/\delta)}{2T_{h,i}(t-1)}}$ with $f^-_{h,i}(t-1)$ set to $-\infty$ if $  T_{h,i}(t) = 0$.
    Following this ranking procedure, we denote the estimated rank of cell $i$ at depth $h$ at time $t$ as $\hat{\left\langle i \right\rangle}_{h,t}$.
After $n$ rounds, \Vroom recommends the element $x(n)$ sampled uniformly from the estimated best cell  $\partition_{h(n),i(n)}$. 
Recommendation of the best cell $\partition_{h(n),i(n)}$ after $n$ rounds is based on $\tilde F_{h,i}(n)- B_{h,i}(n)$ where $B_{h,i}(t)$ defines the confidence bound around our estimate $\tilde F_{h,i}(t)$. For all $t \in [n]$, such a bound is given by: 
\[
\begin{aligned}\label{eq:upperB}
B_{h,i}(t)
=
&f_{\max}\sqrt{2h\bar\log_K(n)\log{2n^2/\delta}\sum^t_{s=1}\hat{\langle i \rangle}_{h,s}}\\
&+f_{\max}\bar\log_K(n)\frac{\log{2n^2/\delta}}{3}.
\end{aligned}
\]


One can view the sampling strategy of \Vroom as a randomised version of that introduced in \StroquOOL\citep{bartlett2019simple}. Essentially, it implements a Zipf exploration~\citep{Powers98AE} meaning that it first ranks the different options (here cells), and then attempts to allocate evaluations inversely proportional to their rank. We note that such a strategy has also been used in previous algorithms, e.g.,  Successive Rejects
(SR) of \cite{audibert2010best} and \Pone~of \cite{abbasi-yadkori2018best}.
To minimise simple regret in the stochastic case, it is crucial to limit the variance of the best-cell estimators. 
Therefore  \Vroom{}, from
its very first pull, chooses with higher probability the cells that are estimated to be among the best. 
This comes with almost no additional cost. Indeed, at depth $h$, pulling 
the estimated best cell with probability 
$1/(h\bar \log_K(n))$ does not prevent from pulling all the cells almost uniformly. More precisely, for any $k\in[K^h]$ all cells ranked below $k$, i.e., $\hat{\left \langle i\right\rangle}_{h,t}\leq k$, are pulled with a probability of at 
least  $1/(hk\bar \log_K(n)).$ 
Therefore, no suboptimal cell is actually left out in 
the early chase for a cell containing $x^\star$. Hence, the variances of the estimators can only increase by a factor of $\bar \log_K(n)$ w.r.t. the uniform strategy.

Additionally, compared to a fixed-phase algorithm,  our
analysis is also more flexible. In fact, we can analyse, for instance, the quality of the estimated ranking $\hat{\langle\cdot\rangle}$ and, consequently, the adaptive sampling procedure of the arms at any round.
Actually, these rounds can be chosen in a problem-dependent fashion, to 
minimise the final regret\footnote{We detail the process by which such rounds are chosen in the sketch of the proof in Section~\ref{sec:rateExp}.}. 

     Remarkably, \Vroom uses a lower confidence bound (LCB) to guide exploration and recommendation. As mentioned earlier, this allows us to compare cells at different and within the same depth by taking into account the uncertainty  $b\sqrt{\frac{\log(n^2/\delta)}{T_{h,i}(n)}}$ around their estimated averages.
For recommendation, this replaces the use of hand-coded cross-validation techniques. For exploration, the use of LCB needs to be handled carefully. For instance, implementing a \emph{pessimism in front of uncertainty} that pulls the cell with the highest LCB would likely result in exclusively pulling one single arm as such bounds increase with the number of pulls. However, LCB are found to combine well with a Zipf sampler that guarantees the estimated $k$ best cells are pulled with an order of $\tcO(n/k)$ almost uniformly.

Interestingly, we demonstrate that potentially biased estimates $\hat f_{h,i} $ can be used to guide exploration as long uniform exploration is guaranteed for all arms. This helps to overcome high variances (in the stochastic case) that the unbiased estimate   $\sum_{t=1}^n\tilde f_{h,i}(t)$ possess and allows us to guaranty that cells containing $x^\star$ are well ranked soon enough. After $n$ rounds, however, we use the unbiased estimate  $\sum_{t=1}^n\tilde f_{h,i}(t)$ to recommend $x(n)$. Being unbiased, our estimates are robust to non-stationary data. Moreover, it is also possible to prove that cells containing a $x^\star$ which have been pulled enough now possess a limited variance in the stochastic setting.

Let us now present our main results for both stochastic and non-stationary data-generating process using \Vroom:

\begin{restatable}[\textcolor{titleTh}{Upper bounds for \Vroom{}}]{theorem}{lightres}    \label{th:UpBOB}

In the non stochastic setting, for any sequence of functions $f_1,\ldots,f_n$ with $f=\frac{1}{n}\sum_{t=1}^n f_t$, we  have,
        after $\timeHorizon$ rounds,     the simple regret of {\Vroom} is bounded as follows:
        \[         
    \Exp \left[r_\timeHorizon\right]
    =
    \cO \left(\log(n/\delta)/n^{\frac{1}{\frac{\log  K}{\log  1/\rho}+2}}\right) 
            \]
        Moreover in the stochastic setting, let $x^\star$  one of the global optimum of $\bar f$ with associated $(\nu,\rho)$, $C>1$, and near-optimality dimension $d=d(\nu,C,\rho)$. Then we  have,
                \[         
            \Exp[ r_n] =    \tcO                 \left(\frac{1}{n}\right)^{\max\left(\frac{1}{d+3}, \frac{1}{\frac{\log  K}{\log  1/\rho}+2} \right)}
        \]
        where the expectation is taken both other $\nu_n$ and the random generation of $f$ with respect to $\bar f$.
\end{restatable}
It is worth noting that the exponent obtained in the stochastic setting is $\max\left(\frac{1}{d+3}, \frac{1}{\frac{\log  K}{\log  1/\rho}+2} \right)$. As mentioned in Section~\ref{sec:ProFo}, in general we have $\frac{1}{d+2}\gg \frac{1}{\frac{\log  K}{\log  1/\rho}+2} $. Therefore in most cases the exponent in the rate of \Vroom is $\frac{1}{d+3}$ and is never worst than the one of uniform allocations $\frac{1}{\frac{\log  K}{\log  1/\rho}+2} $.
\paragraph{Sketch of proof:} 
In the non-stochastic setting, we use the fact that \Vroom pulls at any depth $h$ all the cells  almost uniformly, of order $1/(hK^h)$ up to logarithmic factors, to obtain the same rate as \uniRob.

 For the stochastic case, we face \textbf{Challenge~I}. Indeed, \Vroom uses for recommendation the  estimates $\tilde F_{h,i}(n)$ for cell $\partition_{h,i}$. 
 Consequently, we need to carefully bound the variance of $\tilde F_{h,i}(n)=\sum_{t=1}^n \frac{ f_{h,i,t} 1_{x_t\in \partition_{h,i}}}{ \learnerDist_{h,i,t}}$ for the cells that are near-optimal.
To limit the variance, our algorithm has then two objectives, first identify a deep cell containing $x^\star$ and then pull this cell enough so that the variance of its estimate is low.
Intuitively we follow the idea developed for the stochastic case in Section~\ref{sec:uniSto} that an algorithm which does not know the smoothness parameter $(\nu,\rho)$ can  divide its budget of $n$ rounds into two consecutive parts, one for each objective: First explore $\partition$ for $n^\alpha$ rounds with $\alpha < 1$ in order to build a small number of good candidate cells $\mathcal{C}$ and then secondly cross validate, meaning allocate the rest of the budget, $n-n^\alpha$ rounds, to estimate better and compare this limited number of candidates in $\mathcal{C}$.

 We identify two sources of errors. First the exploration error, $e_r$ is the smallest simple regret among the candidate recommended at the end of the exploration phase after $n^\alpha$ rounds. Following~\cite{locatelli2018adaptivity} we have $e_r=\Omega\left(n^{-\frac{\alpha}{d+2}}\right)$.
The second error, the cross validation error, $e_c$ is the confidence interval of $\tilde F_{h,i}(n)/n$ where our final recommendation $x(n)$ is in cell $\partition_{h,i}$. Assuming  we cannot guaranty the candidates are pulled more than uniformly  during the exploration phase of $n^\alpha$ rounds, we obtain, at time $n$,  $e_c=\cO\left(n^\alpha/n\right)$.\footnote{Alternatively one can bound $e_c$ by recommending with estimates as $\sum_{t=n_\alpha+1}^n \tilde f_{h,i}(t)$ which bias w.r.t. $F_{h,i}(n)$ is  $\cO(n^\alpha)$.}
Simultaneously we want large $\alpha$ to increase the length the exploration phase and reduce the simple regret of our candidates and small $\alpha$ to reduce the variance of our final estimates.
Equaling both source of error we get that $n^{-\alpha/(d+2)}=n^{\alpha-1}$ gives $\alpha = (d+2)/(d+3)$ which leads to a regret $\cO(n^{-\frac{1}{d+3}})$.

\Vroom is implementing implicitly such a strategy without explicitly considering two separate phases and without the knowledge of $d$.
In the stochastic setting,
as discussed above we can 
study the quality of the estimated ranking at any point in time $t\in[n]$.
We divide the $n_\alpha$ in parts $n_1=1,n_2,\ldots,n_{\log(n_\alpha)}$ and analyse the ranking of cell $\partition_{l,i^\star}$ at the end of round $n_l$ for $l\in [n_{\log(n_\alpha)}]$.
To analyse the ranking of the cell $\partition_{h,i}$ we use Lemma~\ref{lem:hstarSto} that provides conditions on $h$ such that we can guarantee
 that after round $n_l$, $t\geq n_l$ the ranking of $\partition_{h,i}$ verifies $\langle i^\star \rangle_{l,t}\lesssim C\rho^{-dl}$.
 Then Lemma~\ref{lem:boundonb} shows that the confidence interval around the average estimate of that cell is $n_\alpha^{\frac{1}{d+2}}$.
%
\section{Related Work}\label{sec:relatedwork}
    \paragraph{BOB} A best of \emph{both} world question has already been addressed by~\cite{abbasi-yadkori2018best} in a more reduced optimisation problem where $\dom=[K]$ is composed of a finite number of $K$ elements known as the best-arm identification (BAI) problem~\citep{Bubeck09PE}. They propose \Pone, an algorithm that achieves, in the stochastic setting, the optimal simple regret rate that any algorithm, with vanishing simple regret in the non-stochastic setting, can achieve.

    \paragraph{Prior work for stochastic and deterministic cases} Among the large work on derivative-free optimisation, we focus on algorithms that perform well under \emph{minimal} assumptions as well as minimal knowledge about the function. 
    While some prior works assume a \emph{global} smoothness of the function~\citep{Pinter13GO,strongin2000global,Hansen2003GO,Kearfott2013RG}, another line of research assumes only a \emph{weak}/\emph{local} smoothness around one global maximum~\citep{auer2007improved,kleinberg2008multi,bubeck2011x}.  
    However, within this latter group, some algorithms require the knowledge of the local smoothness such as \HOO~\citep{bubeck2011x}, \Zooming~\citep{kleinberg2008multi}, or \DOO~\citep{munos2011optimistic}. Among the works relying on an \emph{unknown} local  smoothness,
     \SequOOL~\citep{bartlett2019simple} improves on \SOO~\citep{munos2011optimistic,kawaguchi2016global} 
     and represents the state-of-the-art for the deterministic feedback.
    For the stochastic feedback, \StoSOO~\citep{valko2013stochastic} extends \SOO for a limited class of functions. \POO~\citep{grill2015black-box} and \GPO~\citep{shang2019general} provides more general results. 
    \StroquOOL~\citep{bartlett2019simple} combines up to log factors the  guarantees of \SequOOL and \GPO for deterministic and stochastic feedback respectively without the knowledge of the range of the noise $b$.

\section{Discussion and Future Work}
Our current result holds simultaneously for stochastic and non-stochastic settings. However, it is desirable to also consider the deterministic feedback where evaluations are noiseless and stationary, that is $\forall t\in [n]$, 
    $f_t=f_1$.
    Please refer to the work by \cite{defreitas2012exponential}  for a  motivation, many applications, and references on the importance of this case.
    The question of obtaining the best of the three worlds (BOT) which includes additionally the deterministic setting remains open.
Note that \StroquOOL, for instance, was able to obtain
theoretical guarantees that hold for stochastic and deterministic case settings simultaneously by having a method that adapts to the level of noise $b$ without its knowledge. However, \Vroom requires the knowledge of $b$ and $f_{max}$ to build the lower confidence bound used for recommendation.
To address the BOT question, computing higher moments of our estimates and therefore using concentration inequalities such as the one in the work by \cite{cappe2013kullback} is a potential direction.
We also wonder if a version of \Vroom that is fully using unbiased estimates can solve BOB, while \Vroom uses the $\hat f^-$ estimates to guide exploration. 
and is, therefore, over-fitting the stochastic case.
Finally, fully answering the BOT question may require investigating lower bounds results, a direction we believe is of great interest for future work.

    \bibliographystyle{apalike}
    \bibliography{library,biblio}
    
    \newpage

    \appendix
    \clearpage \section{Proofs of simple regret for the uniform strategies}\label{app:uni}

Results in the deterministic and stochastic cases with known
 smoothness parameters were also reported in~\cite{hren2008optimistic}
 and~\cite{bubeck2010open}.
 
\subsection{Deterministic case}

\paragraph{Deterministic feedback}
Let us consider the uniform exploration that evaluates all the cells at the deepest possible depth $H$ with a budget of $n$ and recommends $x(n)$ the $x$ with the highest observed $f(x)$.  We have $H$ the largest value such that $K^H\leq n$. Therefore $H=\lfloor \log_K(n) \rfloor$.
Because of Assumption~\ref{as:smooth} we have $r_n\leq \nu\rho^{H}$. Therefore $r_n = \cO \left(\left(K/n\right)^{\frac{\log  1/\rho}{\log  K}}\right)$.

\begin{proof} Consider one global optimum $x^\star$. For all $i\in[K^H]$, let $x_{H,i}$ be the element selected for evaluation by the uniform exploration in $\partition_{H,i}$. Then,
\[
	f(x(\timeHorizon)) 
	\stackrel{\textbf{(a)}}{\geq}
	f(x_{H,i^\star_{H}}) 
	\stackrel{\textbf{(b)}}{\geq}
	f(x^\star)- \nu\rho^{H}.
	\]
	where \textbf{(a)} is because uniform has opened all the cells at depth $H$ and $ x(\timeHorizon)=\argmax_{\partition_{H,i}\in \tree}f(x_{H,i})$, and 
	  \textbf{(b)} is by Assumption~\ref{as:smooth}.
Therefore 	 $r_n =	
	f(x^\star) -f(x(\timeHorizon)) \leq \nu\rho^{H} = \nu\rho^{\lfloor \log_K(n) \rfloor} = \nu\rho^{\lfloor \log_K(n/K) +1 \rfloor} \leq \nu\rho^{ \log_K(n/K) } = \nu\left(\left(K/n\right)^{\frac{\log  1/\rho}{\log  K}}\right)   $.
\end{proof}

\subsection{Stochastic  case without knowledge of the smoothness parameters $\nu,\rho$}

\begin{proof}
 Consider one global optimum $x^\star$. For all $i\in[K^H]$, let us  fix $x_{h,i}$ be the element selected for evaluation by the uniform exploration in $\partition_{h,i}$ each of the $\left\lfloor\frac{n}{K^{H}}\right\rfloor$ times this cell is selected.
We define and consider event $\xi_\delta$ and prove 
it holds with high probability.

	Let $\xi_\delta$ be the event under which all average
	estimates in the cells receiving at least one evaluation from uniform are within their classical confidence interval, then  $P(\xi_\delta)\geq 1 - \delta$, where 
	\[
	\xi_\delta
	\triangleq
	\left\{ \forall i \in \left[K^H\right], 
	\left|\hat f_{H,i} - f(x_{H,i})     \right| \leq b\sqrt{    \frac{\log(2\timeHorizon/\delta)}{ n/K^H}}
	\right\}\!\cdot
	\]
We have $P(\xi_\delta)\geq 1 - \delta$,	using Chernoff-Hoeffding's inequality  taking a union bound on all opened cells. On $\xi_\delta$ we have,
\begin{align*}
	f(x(\timeHorizon)) 
	&\stackrel{\textbf{(a)}}{\geq}
	\hat	f(x(\timeHorizon)) -  b\sqrt{\frac{\log(2\timeHorizon/\delta)}{n/K^H}}\\
	&\stackrel{\textbf{(b)}}{\geq}
\hat	f_{H,i^\star} -  \sqrt{\frac{\log(2\timeHorizon/\delta)}{n/K^H}}\\
&		\stackrel{\textbf{(a)}}{\geq}
	f(x_{H,i^\star}) -2 b\sqrt{\frac{\log(2\timeHorizon/\delta)}{n/K^H}}\\
&	\stackrel{\textbf{(c)}}{\geq}
	f(x^\star)- \nu\rho^{H} -2 b\sqrt{\frac{\log(2\timeHorizon/\delta)}{n/K^H}}.
	\end{align*}
	where \textbf{(a)} is because $\xi_\delta$ holds and 
	  \textbf{(b)}  is because uniform has opened all the cells at depth $H$ and $ x(\timeHorizon)=\argmax_{\partition_{h,i}\in \tree}\hat f(x_{h,i})$,
	  and 
	   \textbf{(c)}  is by Assumption~\ref{as:smooth}.
We have $\nu\rho^{H} \leq \nu\left(\left(\frac{K}{n\rho^2}\right)^{\frac{1}{\frac{\log  K}{\log  1/\rho}+2}}\right)  $ and $\sqrt{\frac{\log(2\timeHorizon/\delta)}{n/K^H}}\leq 
\sqrt{\log(2\timeHorizon/\delta)}\left(\left(\frac{K}{n\rho^2}\right)^{\frac{1}{\frac{\log  K}{\log  1/\rho}+2}}\right)$.

Therefore 	 $r_n =	
	f(x^\star) -f(x(\timeHorizon)) \leq \nu\rho^{H} -2 b\sqrt{\frac{\log(2\timeHorizon/\delta)}{n/K^H}}$.
		 $r_n =	\tcO \left(\log(1/\delta)\nu\left(\left(\frac{K}{n\rho^2}\right)^{\frac{1}{\frac{\log  K}{\log  1/\rho}+2}}\right)\right) $
\end{proof}

\subsection{The non-stochastic case}

\order*

\begin{proof}
 Let us fix some depth $H$ and consider a collection of functions $f_1,\ldots,f_n$. Given $f_1,\ldots,f_n$, 
 after $n$ rounds the random variables $\tilde f_{H,i}(t)$ are conditionally independent from each other for all $i$ at depth $H$ and for all $t\in[n]$  as we have $ \Pro(\pulledArm_t\in\partition_{H,i}  \cap h\geq 0) =\Pro(\pulledArm_t\in\partition_{H,i}) \geq 1/K^{H}$ are fixed for all $ i$ at depth $H$ and $t \in [n]$.

 The variance of $\tilde f_{H,i}(t)$ is the variance of a scaled Bernoulli random variable with parameter $\Pro(\pulledArm_t\in\partition_{H,i}) \geq 1/K^{H}$ and range
$\left[0, K^H E_{x\sim U(\partition_{h,i})}\left[ f_t(x)\right]\right]$, therefore we have $|\tilde f_{H,i}(t) - E_{x\sim U(\partition_{h,i})}\left[ f_t(x)\right]| \leq K^H$
, and $\sigma^2_{\tilde f_{H,i}(t) - E_{x\sim U(\partition_{h,i})}\left[ f_t(x)\right]}
=  \sigma^2_{\tilde f_{H,i}(t)}
\leq 1/K^H(1 - 1/K^H)K^{2H} \tilde f^2_{H,i}(t) \leq
K^H f^2_{max}$.

We define and consider event $\xi_\delta$ and prove 
it holds with high probability.
	Let $\xi_\delta$ be the event under which all average
	estimates in all the cells at depth $H$ are within their classical confidence interval, then  $P(\xi_\delta)\geq 1 - \delta$, where 
	\[
	\begin{aligned}
	\xi_\delta
	\triangleq &
	\left\{ \forall \partition_{H,i},~ 
	\left|\tilde F_{H,i}(n) - F_{H,i}(n)     \right| \right.\\
&	\left.
\leq \sqrt{ 2n f^2_{max}K^H   \log(\timeHorizon^2/\delta) } +\frac{ f^2_{max}}{3}K^H\log(\timeHorizon^2/\delta)
	\right\}\,\cdot
	\end{aligned}
	\]
We have $P(\xi_\delta)\geq 1 - \delta$,	using Bennett's inequality from Theorem 3 in~\cite{maurer2009empirical} and from taking a union bound on all opened cells. 
We denote $B_h=\sqrt{ 2nK^h   \log(\timeHorizon^2/\delta) } +bK^h\log(\timeHorizon^2/\delta)$ and we denote by 
 $h(n)$ the depth of $x(\timeHorizon)$.
On $\xi_\delta$ we have, for any $H\in [\left\lfloor\log_K(n)\right\rfloor]$,
\begin{align*}
\Exp [f(x(\timeHorizon))]  
	&\stackrel{\textbf{(a)}}{\geq}
	\frac{1}{n}\left(\tilde	F(x(\timeHorizon)) - B_{h(n)}\right)
	\stackrel{\textbf{(b)}}{\geq}
\frac{1}{n}\left(\tilde	F_{H,i^\star} - B_H\right)\\
&		\stackrel{\textbf{(a)}}{\geq}
\frac{1}{n}\left(	F_{H,i^\star} -2 B_H\right)\\
&	\stackrel{\textbf{(c)}}{\geq}
	f(x^\star)- \nu\rho^{H} -2 B_H/n.\numberthis\label{eq:lower}
	\end{align*}
	where \textbf{(a)} is because $\xi_\delta$ holds 
	  \textbf{(b)}  is by definition of $x(n)$ as $x(n) \gets  \argmax\limits_{x_{h,i}} \tilde F_{h,i}(n) - B_h$,
	 and 
	   \textbf{(c)}  is by Assumption~\ref{as:smooth}.
	   
In order to maximize the lower bound in~\ref{eq:lower} we set $H=\left\lfloor\log_{K/\rho^2}(n)\right\rfloor$.
We have $\nu\rho^{H} \leq \nu\left(\left(\frac{K}{n\rho^2}\right)^{\frac{1}{\frac{\log  K}{\log  1/\rho}+2}}\right)  $ and $\sqrt{\log(\timeHorizon^2/\delta)K^H/n}\leq 
\sqrt{\log(\timeHorizon^2/\delta)}\left(\left(\frac{K}{n\rho^2}\right)^{\frac{1}{\frac{\log  K}{\log  1/\rho}+2}}\right)$
and $bK^H/n\log(\timeHorizon^2/\delta)=\cO\left( 
\sqrt{\log(\timeHorizon^2/\delta)}\left(\left(\frac{K}{n\rho^2}\right)^{\frac{2}{\frac{\log  K}{\log  1/\rho}+2}}\right)\right)$.

Therefore 	 $\Exp_{\nu_n}	[r_n]  =	
	f(x^\star) - \Exp[ f(x(\timeHorizon)) ]\leq \nu\rho^{H+1} -2B/n$.
		 $r_n =	\cO \left(\log(n/\delta)\left(\frac{K}{n\rho^2}\right)^{\frac{1}{\frac{\log  K}{\log  1/\rho}+2}}\right) $

\end{proof}

\section{Proofs of simple regret for \Vroom}\label{app:vroom}

\paragraph{The non-stochastic feedback case}

\begin{proof}
 Let us fix some depth $H$ and consider a collection of functions $f_1,\ldots,f_n$.
	 Given $f_1,\ldots,f_n$, 
 after $n$ rounds the random variables $\tilde f_{H,i}(t)$ 
 can 
	be dependent of each other for all $h\geq 0$ and $i\in[K^H]$ 
	and $t\in[\timeHorizon]$ as $\learnerDist_{h,i,t}$ 
	depends on previous observations at previous rounds. 
	Therefore, we use the Bernstein inequality for martingale  differences by~\citet{Freedman75OT}.



 The variance of $\tilde f_{H,i}(t)$ is the variance of a scaled Bernoulli random variable with parameter $\Pro(\pulledArm_t\in\partition_{H,i}) \geq 1/K^{H} \log^2_K(n)$ and range
 $\left[0, K^H E_{x\sim U(\partition_{h,i})}\left[ f_t(x)\right)\log^2(n)\right]$,


therefore we have $|\tilde f_{H,i}(t) - E_{x\sim U(\partition_{h,i})}\left[ f_t(x)\right]| \leq K^H\log^2_K(n)f_{max} $
, and $\sigma^2_{\tilde f_{H,i}(t) - E_{x\sim U(\partition_{h,i})}\left[ f_t(x)\right]}
=  \sigma^2_{\tilde f_{H,i}(t)}
\leq 1/K^H(1 - 1/K^H)K^{2H} \tilde f^2_{H,i}(t) \leq
K^H f^2_{max}$.

	Then, following the same reasoning as in the proof 
	of Theorem~\ref{th:unilcbad}, but replacing the Bernstein inequality by the Bernstein inequality for martingale differences of~\cite{Freedman75OT} applied to the martingale 
	differences $\tilde f_{k,t}-\tilde f_{k,t}$,  
	we obtain the claimed result for the adversarial case.
\end{proof}

		\paragraph{The i.i.d. stochastic feedback case}

\begin{proof}


Note that as the regret guaranties proved in the non-stochastic case also hold in the stochastic case. So we are left to prove 
	$	\Exp[ r_n] =	\tcO 				\left(\frac{1}{n}\right)^{\frac{1}{d+3}}$.

		We place ourselves in the i.i.d.\,stochastic setting described 
	in Section~\ref{s:intro}.
	Let us consider a fixed depth $H$ which value will be chosen towards the end of the proof in order to minimize the simple regret with respect to this $H$.

We consider one global optimum $x^\star$  of $\bar f$ with associated $(\nu,\rho)$, $C>1$, and near-optimality dimension $d=d(\nu,C,\rho)$.

We define $n_\alpha\in [n]$ and will analyze how \Vroom explore the depth $h\leq\left\lfloor \log_K(n_\alpha)\right\rfloor$.

First, we define  the rounds used for comparisons.

%
We define the times $n_h= \beta n_\alpha\frac{\sum_{h'=1}^{h} \frac{1}{h'}}{ \log_K(n_\alpha)}$ for $h \in \left\lfloor \log_K(n_\alpha)\right\rfloor$ and where $\beta>1$ is a constant that we will fix later such that $n_h\leq n$.
 To ease the notation and without loss of generality, for each depth $h$, we  assume that 
	the cells are sorted by their means so that cell $1$ is the best,
	$\bar f_{h,1}\geq \bar f_{h,2}\geq\ldots\geq \bar f_{h,K^h}$.

We define and consider event $\xi_\delta$ and prove 
it holds with high probability.

	Let $\xi_\delta$ be the event under which all average
	estimates in all the cells at depth $H$ are within their classical confidence interval, then  $P(\xi_\delta)\geq 1 - \delta$, where $\xi_\delta$ is decomposed in three sub-events $\xi_\delta=\xi^1_\delta\cap\xi^2_\delta\cap\xi^3_\delta$ where
	\begin{align*}
	\xi^1_\delta
	\triangleq
	&\left\{  \forall \partition_{h,i}, h\leq\left\lfloor \log_K(n)\right\rfloor: \right.\\
	&\quad\left|\tilde F_{h,i}(n) - n \bar f_{h,i}     \right| \leq B^{adv}_{h,i}(n)	\\
	&\quad \text {and }\left.\left|\tilde F_{h,i}(n) -  F_{h,i}(n)    \right| \leq B^{adv}_{h,i}(n)
	\right\}\!\CommaBin	\\
		\xi^2_\delta
	\triangleq
&	\left\{ \forall \partition_{h,i}, h\leq\left\lfloor \log_K(n)\right\rfloor, \forall t\in[n],\right.\\ 
	&\quad		\left.\left|\hat f_{h,i}(t) -  \bar f_{h,i}     \right| \leq B^{iid}_{h,i}(t)
	\right\}\!\CommaBin\\
		\xi^3_\delta
	\triangleq
	&\left\{ \forall  h\leq\left\lfloor\log_K(n_\alpha)\right\rfloor,\right.\\ &\quad\left. \forall t\geq 8n_\alpha\log^3(n),
	T_{h,i^\star}(t)\geq 
		\Exp \left[\frac{T_{h,i^\star}(t)
		}{2}\right]
	\right\}\!\cdot
	\end{align*} 
We have $P(\xi^1_\delta)\geq 1 - \delta/2$. Indeed to bound  $\left|\tilde F_{h,i}(n) - n \bar f_{h,i}     \right| $ we use  the Bernstein inequality for martingale differences of~\cite{Freedman75OT} applied to the martingale 
	differences $\tilde f_{k,t}-\bar f_{k,t}$ and from taking a union bound on all cells at depth $h\leq \left\lfloor\log_K(n)\right\rfloor$. 
	We have $P(\xi^1_\delta)\geq 1 - \delta/2$. Indeed, to bound  $	\left|\hat f_{h,i}(t) -  \bar f_{h,i}     \right| $ we use  the Chernoff-Hoeffding inequality and take a union bound on all cells at depth $h\leq \left\lfloor\log_K(n)\right\rfloor$. 
Finally we have $P(\xi^3_\delta)\geq 1 - \log(n)/n$. Indeed, 
using a Chernoff bound we have for $\forall  h\leq\left\lfloor \log_K(n_\alpha)\right\rfloor, \forall t\geq 8n_\alpha\log^3(n)$, 
\begin{align*}
&\Pro\left(	T_{h,i^\star}(t) \leq 
			\Exp \left[\frac{T_{h,i^\star}(t)
		}{2}\right]\right) \\
&		\leq \exp\left(-\frac{1}{8} 	\Exp \left[\frac{T_{h,i^\star}(t)
		}{2}\right]\right)\\
&		\leq \exp\left(-\frac{1}{8} \Exp \left[\sum^{t-1}_{s=1} \Pro(\pulledArm_s\in\partition_{h,i^\star})\right]\right)\\
	&	\stackrel{\textbf{(a)}}{\leq} \exp\left(-\frac{1}{8} \sum^{t-1}_{s=1}   \frac{ 1 }{n_\alpha  \log^2(n) ~}\right)\\
&		\leq \exp\left(-\frac{1}{8}    \frac{ 8 n_\alpha  \log^3(n) }{n_\alpha  \log^2(n) ~}\right)
		=\exp\left(-\log(n)    \right)=\frac{1}{n}
\end{align*}
where \textbf{(a)} is because 
$
\Pro(\pulledArm_t\in\partition_{h,i^\star})
\geq \learnerDist_{h,i^\star,t} 
=  	  \frac{ 1 }{h\hat{\langle i^\star \rangle}^{\vphantom{X}}_{h,t} \bar \log_K(n) ~} 
\geq    \frac{ 1 }{\log_K n_\alpha K^h \bar \log_K(n) ~}
\geq    \frac{ 1 }{\log_K n_\alpha n_\alpha \bar \log_K(n) ~}
\geq    \frac{ 1 }{n_\alpha  \log^2_K(n) ~}
\geq    \frac{ 1 }{n_\alpha  \log^2(n) ~}$.

 We can therefore decompose the regret $r_n$ as
\begin{align*}
\Exp[r_n] &= \left(\delta+\frac{\log(n)}{n}\right) \Exp[r_n|\eventbob_\delta^c]+ \left(1-\delta-\frac{\log(n)}{n}\right)\Exp[r_n|\eventbob_\delta]\\
 &\leq \left(\delta+\frac{\log(n)}{n}\right)  f_{\max}+ \Exp[r_n|\eventbob_\delta].\numberthis\label{takingafarm}
\end{align*}

 As  we will set $\delta=\frac{4b}{f_{\max} \sqrt{n}}$ the first term of Inequality~\ref{takingafarm} is already smaller than the claimed result of the Theorem so we now focus on bounding the second term.

%


For any   $x^\star\!\!,$   we write
\[\depthOp_{h}=\left\{ h'\geq0: \forall t\geq n_h, \hat{ \left\langle \partition_{h',i^\star} \right \rangle}_{h',t} \leq C\rho^{-dh'}\right\}\] that contains all the depth $h$ such that for all time $t\geq n_h$ the cell containing~$x^\star$ at depth $h$ 
 is ranked with a smaller index than $C\rho^{-dh}$ by \Vroom. As explained above we are trying here to introduce tools that will help us to upper bound the ranking of the best arm to be able then to upper bound the variance of its estimates.

On $\xi_\delta$ we have, for all $H\in[\left\lfloor\log_K(n)\right\rfloor]$ 
%
%
\begin{align*}
	\Exp [f(x(\timeHorizon))] 
	&\stackrel{\textbf{(a)}}{\geq}
	\frac{1}{n}\left(\tilde	F_{h(\timeHorizon),i(\timeHorizon)} - B_{h(\timeHorizon),i(\timeHorizon)}(n)\right)\\
&	\stackrel{\textbf{(b)}}{\geq}
\frac{1}{n}\left(\tilde	F_{\depthOp_{H},1} - B_{\depthOp_{H},1}(n)\right)\\
&		\stackrel{\textbf{(a)}}{\geq}
\frac{1}{n}\left(	n\bar f_{\depthOp_{H},1} -2 B_{\depthOp_{H},1}(n)\right)\\
&	\stackrel{\textbf{(c)}}{\geq}
	\bar f(x^\star)- \nu\rho^{\depthOp_{H}} -2 B_{\depthOp_{H},1}(n)/n.\numberthis\label{eq:lower2}
	\end{align*}
	where \textbf{(a)} is because $\xi_\delta$ holds
	  \textbf{(b)}  is by definition of $x(n)$ as $x(n) \gets  \argmax\limits_{x_{h,i}} \tilde F_{h,i}(n) - B_{h,i}(n)$,
	 and 
	   \textbf{(c)}  is by Assumption~\ref{as:smooth}.

We now need to bound $\depthOp_{H}$ and bound $B_{\depthOp_{H},1}(n)$ for some $H\in [\left\lfloor\log_K(n_\alpha)\right\rfloor]$. To obtain a tight bound we try to have $\nu\rho^{\depthOp_{H}} $ and  $B_{\depthOp_{H},1}(n)$ of the same order.


We use for that Lemma~\ref{lem:hstarSto} that provide sufficient condition in Equation~\ref{ga:suff} to lower bound  $\depthOp_{H}$. We now define the quantity $\tilde h$ that verify this condition.
$\tilde h  $ is so that the $\nu\rho^{\depthOp_{H}}$ and  $B^{iid}_{\depthOp_{H},1}(n)$ are equal.
	We denote $\tilde h$  the real number satisfying
	\begin{equation}\label{eq:eqStro}
	\frac{ n_\alpha\nu^2\rho^{2\tilde h}}{K\tilde hb^2\log^2(2\timeHorizon^2/\delta)}
	=
	C\rho^{-d\tilde h}.  
	\end{equation}
	Our approach is to solve Equation~\ref{eq:eqStro} and then verify that it gives a valid indication of the behavior of our algorithm in term of its optimal  $h$.     We have  
	\[\tilde h = \frac{1}{(d+2)\log(1/\rho)}\lambertW\left(\frac{\nu^2 n_\alpha (d+2)\log(1/\rho)}{K C b^2\log^2(2\timeHorizon^2/\delta)}\right)\]
	where standard $\lambertW$ is the Lambert $\lambertW$ function.



	Using standard properties of the $\lfloor\cdot\rfloor$ function, 
	we have 
	%
	\begin{align*}
&		\frac{ n_\alpha\nu^2\rho^{2\left\lfloor\tilde h\right\rfloor}}{K\left\lfloor\tilde h\right\rfloor b^2\log^2(2\timeHorizon^2/\delta)}
		\geq
		\frac{ n_\alpha\nu^2\rho^{2\tilde h}}{K\tilde hb^2\log^2(2\timeHorizon^2/\delta)}\\
&		=
		C\rho^{-d\tilde h} 
			\geq C\rho^{-d \left\lfloor\tilde h\right\rfloor}.
	\end{align*}
	From the previous inequality we also have, as $d\leq \log(K)/\log(1/\rho)$,
	\begin{align*}
		&n_\alpha
		\geq\\\nonumber
		&\frac{ n_\alpha\nu^2\rho^{2\left\lfloor\tilde h\right\rfloor}}{K\left\lfloor\tilde h\right\rfloor b^2\log^2(2\timeHorizon^2/\delta)}
		\geq C\rho^{-d \left\lfloor\tilde h\right\rfloor}
		\geq K^{ \left\lfloor\tilde h\right\rfloor}.
	\end{align*}
	which leads to $\left\lfloor\tilde h\right\rfloor\leq \left\lfloor \log_K(n_\alpha)\right\rfloor$.
	Having $\left\lfloor\tilde h\right\rfloor\in [\left\lfloor \log_K(n_\alpha)\right\rfloor]$ and using Lemma~\ref{lem:hstarSto} we have that if  $\beta\geq 8\log^3(n)\left\lfloor \log_K(n_\alpha)\right\rfloor$ then 
	$\depthOp_{\left\lfloor\tilde h\right\rfloor}\geq \left\lfloor\tilde h\right\rfloor$.
	
To bound $B_{\depthOp_{H},1}(n)$ we use Lemma~\ref{lem:boundonb}. Therefore, choosing $H=\left\lfloor\tilde h\right\rfloor$, we get to rewrite Equation~\ref{eq:lower2} as
\begin{align*}
	\Exp [f(x(\timeHorizon))] 
&\geq
	\bar f(x^\star)- \nu\rho^{\depthOp_{\left\lfloor\tilde h\right\rfloor}} -2 B_{\depthOp_{\left\lfloor\tilde h\right\rfloor},1}(n)/n\\
 &	\geq 	\bar f(x^\star)- \nu\rho^{\left\lfloor\tilde h\right\rfloor} \\
 &\quad-   4f_{\max}\sqrt{\log^3(2n^2/\delta)\left( \frac{n_\alpha^2}{n^2}    +      \frac{C\rho^{-d\depthOp_{\left\lfloor\tilde h\right\rfloor} }}{n}\right) } \numberthis\label{eq:alliknow}
\end{align*}
		Moreover, as proved by~\citet{Hoorfar08IO}, the Lambert $W(x)$ function verifies for $x\geq e$,
	$\lambertW(x)\geq \log\left(\frac{x}{\log x}\right)$.
	Therefore, if $\frac{\nu^2 n_\alpha (d+2)\log(1/\rho)}{K C b^2\log(2\timeHorizon^2/\delta)}>e$ we have, 
	we have the first term in Equation~\ref{eq:alliknow}
\begin{align*}
    \rho^{\left\lfloor\tilde h\right\rfloor}
  &  \leq 
    \rho^{   \frac{1}{(d+2)\log(1/\rho)}\lambertW\left(\frac{\nu^2 n_\alpha (d+2)\log(1/\rho)}{K C b^2\log^2(2\timeHorizon^2/\delta)}\right) -1}\\
   &  \leq 
    \rho^{   \frac{1}{(d+2)\log(1/\rho)}\log\left(\frac{\frac{\nu^2 n_\alpha (d+2)\log(1/\rho)}{K C b^2\log^2(2\timeHorizon^2/\delta)}}{e\log^2\left(\frac{\nu^2 n_\alpha (d+2)\log(1/\rho)}{K C b^2\log^2(2\timeHorizon^2/\delta)}\right)}\right) }\\
&    =
    \left(\frac{\frac{\nu^2 n_\alpha (d+2)\log(1/\rho)}{K C b^2\log^2(2\timeHorizon^2/\delta)}}{e\log\left(\frac{\nu^2 n_\alpha (d+2)\log(1/\rho)}{K C b^2\log^2(2\timeHorizon^2/\delta)}\right)}\right)^{\frac{-1}{(d+2)}}
\end{align*}
Then we have, from Equation~\ref{eq:eqStro},
\begin{align*}
\sqrt{\frac{C\rho^{-d\depthOp_{\left\lfloor\tilde h\right\rfloor}}}{n}} 
&\leq
\sqrt{\frac{C\rho^{-d\tilde h}}{n}} 
= 	
\sqrt{\frac{ n_\alpha\nu^2\rho^{2\tilde h}}{nK\tilde hb^2\log^2(2\timeHorizon^2/\delta)}}\\
&\leq 
\frac{ \nu\rho^{\tilde h} }{\sqrt{K}b\log^2(2\timeHorizon^2/\delta)},
\end{align*}
which is bounded  above.

Then in Equation~\ref{eq:alliknow}, using that $\sqrt{a'+b'}\leq \sqrt{a'}+\sqrt{b'}$ for two non negative numbers $(a',b')$, we have three terms of the shape:
$ n_\alpha^{\frac{-1}{d+2}}+n_\alpha/n +n_\alpha^{\frac{-1}{d+2}}$.
As explained in the sketch of proof we need to have $n_\alpha$ of order $n^{\frac{d+2}{d+3}}$ in order to minimize the previous sum.

More precisely we set $n_\alpha = n^{\frac{d+2}{d+3}} / (8 \log^4(n))$ and set $\beta = 8 \log^4(n)$ and $\delta=\frac{4b}{f_{\max} \sqrt{n}}$ and obtain the claimed result.

\end{proof}


\begin{restatable}{lemma}{restaboundonb}\label{lem:boundonb}
If $\beta\geq 8\log^4(n)\left\lfloor \log_K(n_\alpha)\right\rfloor$,	for any global optimum $x^\star$ with associated $(\nu,\rho)$ from Assumption~\ref{as:smooth}, any $C>1$, for any $\delta \in (0,1)$, on  event $\xi_\delta$
 defined above,
	 for any depth $h\in [\left\lfloor \log_K(n_\alpha)\right\rfloor]$, we have that if 
	 \begin{gather}\label{ga:suff2}
	     \frac{ n_\alpha}{K}\nu^2\rho^{2h}/(b^2h\log^2(2n^2/\delta))
	\geq  C\rho^{-d(\nu,C,\rho)h},
	 \end{gather}
	  that
	\[
B_{\depthOp_{h},1}(n) \leq   2f_{\max}\sqrt{\log^3(2n^2/\delta)( n_\alpha^2
    +    n  C\rho^{-d\depthOp_{h}})} .
\]
\end{restatable}
\begin{proof}
The assumptions of Lemma~\ref{lem:hstarSto} being verified we have $h\in\depthOp_{h}$. 
Also we have, 
    \begin{align}\label{eq:leonora}
&        B_{\depthOp_{h},1}(n) \\
        &=  f_{\max}\sqrt{2\depthOp_{h}(n)\bar\log_K(n)\log{2n^2/\delta}\sum^n_{s=1}\hat{\langle 1 \rangle}_{h,s}
}\\
&+f_{\max}\bar\log_K(n)\frac{\log{2n^2/\delta}}{3}.
    \end{align}
    We bound the first term  by having 
   \begin{align*}
    \sum^n_{s=1}\hat{\langle 1 \rangle}_{h,s} 
   & =
    \sum^{\left\lfloor \log_K(n_\alpha)\right\rfloor-1}_{h=0} \sum^{n_h}_{s=n_h+1}\hat{\langle 1 \rangle}_{h,s}  \\  
&\quad    +  \sum^{n}_{s=n_{\left\lfloor \log_K(n_\alpha)\right\rfloor}+1}\hat{\langle 1 \rangle}_{h,s}  \\
    & \stackrel{\textbf{(a)}}{\leq}
    \sum^{\left\lfloor \log_K(n_\alpha)\right\rfloor-1}_{h=0} \sum^{n_h}_{s=n_h+1}K^{\left\lfloor \log(n_\alpha)\right\rfloor} \\  
&\quad
    +  \sum^{n}_{s=n_{\left\lfloor \log_K(n_\alpha)\right\rfloor}+1}C\rho^{-d\depthOp_{h}}\\
   &     \leq
    \sum^{\left\lfloor \log_K(n_\alpha)\right\rfloor-1}_{h=0} \sum^{n_h}_{s=n_h+1}n_\alpha
    +    n  C\rho^{-d\depthOp_{h}}\\
       &     \leq
    n_\alpha^2
    +    n  C\rho^{-d\depthOp_{h}}
   \end{align*}
    where \textbf{(a)} is because  $h\in [\left\lfloor \log_K(n_\alpha)\right\rfloor]$ and $\depthOp_{h}\geq h$.
    
    Because in Equation~\ref{eq:leonora} the second term is smaller than the first, we have 
    \begin{align}\label{eq:alienware}
        &B_{\depthOp_{h},1}(n) \\
        &=  2f_{\max}\sqrt{2\log_K(n_\alpha)\bar\log_K(n)\log(2n^2/\delta)( n_\alpha^2
    +    n  C\rho^{-d\depthOp_{h}}) } \\
    &\leq   2f_{\max}\sqrt{2\log^3(2n^2/\delta)( n_\alpha^2
    +    n  C\rho^{-d\depthOp_{h}})
}.
    \end{align}
\end{proof}

\begin{restatable}{lemma}{restahstarSto}\label{lem:hstarSto}
If $\beta\geq 8\log^4(n)\left\lfloor \log_K(n_\alpha)\right\rfloor$,	For any global optimum $x^\star$ with associated $(\nu,\rho)$ from Assumption~\ref{as:smooth}, any $C>1$, for any $\delta \in (0,1)$, on  event $\xi_\delta$
 defined above,
	 for any depth $h\in [\left\lfloor \log_K(n_\alpha)\right\rfloor]$, we have that if 
	 \begin{gather}\label{ga:suff}
	     \frac{ n_\alpha}{K}\nu^2\rho^{2h}/(b^2h\log^2(2n^2/\delta))
	\geq  C\rho^{-d(\nu,C,\rho)h},
	 \end{gather}
	  that
	$h\in\depthOp_{h}$.
\end{restatable}
\begin{proof}To simplify notation we write $d(\nu,C,\rho)$ as $d$.
	We place ourselves on  event $\xi_\delta$ defined above. 
	We prove the statement of the lemma, given that event $\xi_\delta$ holds, by induction in the following sense. For a given $h$, we assume the hypotheses of the lemma for that $h$ are true and we prove by induction that $h'\in\depthOp_{h'}$ for $h'\in[h]$. \\[.1cm] 
	$1^\circ$ For $h'= 0$, we trivially have that $0\in\depthOp_{h'} $.\\     
	$2^\circ$  Now consider $h'>0$, and assume $ h'-1\in\depthOp_{h'-1}$ with the objective to prove that $h'\in\depthOp_{h'}$.
	Therefore, for all $t\geq n_{h'-1}$,
	$\hat{\left\langle \partition_{h'-1,i^\star}\right \rangle}_{h'-1,t} \leq C\rho^{-d(h'-1)}$. 
 
	For the purpose of contradiction, let us assume that their exists $t\geq n_{h'}$, such that $\hat{\left\langle \partition_{h',i^\star}\right \rangle}_{h',t} > C\rho^{-dh'}$.
	This would mean that there exist at least 
	$C\rho^{-dh'}$ cells from $\left\{\partition_{h',i}\right\}$, distinct from $\partition_{h',i_h^\star}$, satisfying 
	$\hat f^-_{h',i}(t)
	\geq
	\hat f^-_{h',i^\star_{h'}}(t)$.
	%
	This means that, for these cells we have
\[	\begin{aligned}
	\bar f_{h',i}
	&\stackrel{\textbf{(b)}}{\geq} 
	\hat f^-_{h',i}(t)
	\geq
	\hat f^-_{h',i^\star_{h'}}(t)
	\stackrel{\textbf{(b)}}{\geq} 
\bar	f_{h',i^\star_{h'}}(t)- 2b\sqrt{    \frac{\log(2\timeHorizon^2/\delta)}{    2T_{h',i^\star_{h'}}(t) }}\\
	&	\stackrel{\textbf{(c)}}{\geq} 
\bar	f_{h',i^\star_{h'}}(t)- 2b\sqrt{    \frac{\log(2\timeHorizon^2/\delta)}{    \frac{\beta n_\alpha}{h\left\lfloor \log_K(n_\alpha)\right\rfloor CK\rho^{-d{h'-1}} }}}\\
&\geq 
\bar	f_{h',i^\star_{h'}}(t)- 2b\sqrt{    \frac{\log(2\timeHorizon^2/\delta)}{    \frac{n_\alpha}{h\left\lfloor \log_K(n_\alpha)\right\rfloor CK\rho^{-d{h'}} }}}\\
	&	\stackrel{\textbf{(d)}}{\geq} 
\bar	f_{h',i^\star_{h'}}- 2\nu\rho^{h}
		\geq 
	\bar f_{h',i^\star_{h'}}- 2\nu\rho^{h'},
	\end{aligned}
	\]
	where 
	\textbf{(b)} is because $\xi_\delta$ holds, 
	 \textbf{(d)} is because  by assumption (Equation~\ref{ga:suff}) of the lemma, for $h'\in [h]$,
	 $   \frac{ n_\alpha}{K}\nu^2\rho^{2h'}/(b^2h\log^2(2n^2/\delta))
	\geq   \frac{ n_\alpha}{K}\nu^2\rho^{2h}/(b^2h\log^2(2n^2/\delta))
	\geq  C\rho^{-dh}\geq  C\rho^{-dh'}$.
	\textbf{(c)} is because on $\xi_\delta$, as $\beta\geq 8\log^3(n)\left\lfloor \log_K(n_\alpha)\right\rfloor$ and $h\leq \left\lfloor \log_K(n_\alpha) \right\rfloor$,
$	 \forall t\geq  n_h= \beta n_\alpha\frac{\sum_{h'=1}^{h} \frac{1}{h'}}{ \log_K(n_\alpha)}  \geq 8 n_\alpha\log^3(n)$,  have
\begin{align*}
	T_{h',i^\star_{h'}}(t)
	&\geq 
		 \Exp \left[\sum^{t-1}_{s=1} \frac{\Pro(\pulledArm_s\in\partition_{h',i^\star})}{2}\right]\\
		 &\geq	 \Exp \left[\sum^{n_{h'}}_{s=n_{h'-1}} \frac{\Pro(\pulledArm_s\in\partition_{h',i^\star})}{2}\right]\\
		&	\stackrel{\textbf{(e)}}{\geq} 	\sum^{n_{h'}}_{s=n_{h'-1}} \frac{1}{  2CK\rho^{-d{h'-1}} }\\
		&\geq \beta \frac{n_\alpha}{ 2h\left\lfloor \log_K(n_\alpha)\right\rfloor CK\rho^{-d{h'-1}} },
		\end{align*}
		where \textbf{(e)} is because we have $\langle \partition_{h'-1,i^\star}\rangle_{h'-1,t} \leq C\rho^{-d(h'-1)}$ which gives $\Pro(\pulledArm_t\in\partition_{h,i})\geq \frac{1}{KC\rho^{-d(h'-1)}}$
	%
	as $f_{h',i^\star_{h'}} \geq f(x^\star) -\nu\rho^{h'}$ by Assumption~\ref{as:smooth}, it follows that $ \mathcal N_{h'}(3\nu\rho^{h'})> \left\lfloor C\rho^{-dh'}\right\rfloor$. 
	This leads to having a contradiction with the function $f$ being of near-optimality dimension $d$
	as defined in~Definition~\ref{def:neardim}.
	Indeed, the condition $\mathcal N_{h'}(3\nu\rho^{h'}) \leq     C\rho^{-dh'}$ in 
	Definition~\ref{def:neardim} is equivalent to the condition
	$\mathcal N_{h'}(3\nu\rho^{h'}) \leq \left\lfloor    C\rho^{-dh'}\right\rfloor$ as  $\mathcal N_{h'}(3\nu\rho^{h'})$ is an integer. 
	Reaching the contradiction proves the claim of the lemma.
\end{proof}

\end{document}